\documentclass{scrartcl} 

\usepackage[utf8]{inputenc} %
\usepackage[T1]{fontenc}    %
\usepackage{amsfonts}       %
\usepackage{amsmath}
\usepackage{amssymb}
\usepackage{mathtools}
\usepackage{amsthm}
\usepackage{bm}
\usepackage{color}
\usepackage{soul}           %
\usepackage{url}
\usepackage{booktabs}       %
\usepackage{nicefrac}       %
\usepackage{enumerate}
\usepackage{graphicx}
\DeclareGraphicsExtensions{.pdf,.png,.jpg}
\graphicspath{{figures/}}
\usepackage{comment}
\usepackage{hyperref}
\usepackage{centernot}     %

\usepackage[titletoc,title]{appendix}
\usepackage{fullpage}

\newcommand{\R}{\mathbb{R}}

\newcommand{\BigO}[1]{\ensuremath{\mathcal{O}\left(#1\right)}}                  %
\newcommand{\BigOm}[1]{\ensuremath{\Omega\left(#1\right)}}                      %

\newcommand{\vect}[1]{\ensuremath{\mathbf{#1}}}                                 %
\newcommand{\vectsym}[1]{\ensuremath{\boldsymbol{#1}}}                          %
\newcommand{\mat}[1]{\ensuremath{\mathbf{\MakeUppercase{#1}}}}                  %
\newcommand{\KL}[2]{\ensuremath{\mathbb{KL} \left(#1 \middle\Vert #2 \right)}}   %
\newcommand{\Exp}[2]{\ensuremath{\mathbb{E}_{#1}\left[#2\right]}}                %
\newcommand{\Var}[2]{\ensuremath{\mathrm{Var}_{#1}\left[#2\right]}}                %
\newcommand{\Cov}[2]{\ensuremath{\mathrm{Cov}_{#1}\left[#2\right]}}              %
\newcommand{\Ind}[1]{\ensuremath{\mathbf{1}\left[#1\right]}}                     %
\newcommand{\NormI}[1]{\ensuremath{\left\lVert #1 \right\rVert}_1}               %
\newcommand{\NormII}[1]{\ensuremath{\left\lVert #1 \right\rVert}_2}              %
\newcommand{\NormInfty}[1]{\ensuremath{\left\lVert #1 \right\rVert_{\infty}}}    %

\newcommand{\matrx}[1]{\begin{bmatrix}#1\end{bmatrix}}                           %

\newcommand{\InNorm}[1]{{\left\vert\kern-0.2ex\left\vert\kern-0.2ex\left\vert #1 
    \right\vert\kern-0.2ex\right\vert\kern-0.2ex\right\vert}}                    %

\newcommand{\InNormII}[1]{{\left\vert\kern-0.2ex\left\vert\kern-0.2ex\left\vert #1 
    \right\vert\kern-0.2ex\right\vert\kern-0.2ex\right\vert}_2}                    %

\newcommand{\InNormInfty}[1]{{\left\vert\kern-0.2ex\left\vert\kern-0.2ex\left\vert #1 
    \right\vert\kern-0.2ex\right\vert\kern-0.2ex\right\vert}_{\infty}}           %

\newcommand{\Abs}[1]{\ensuremath{\left \lvert #1 \right \rvert}}                 %
\newcommand{\Prob}[1]{\ensuremath{\mathrm{Pr}\left\{ #1 \right\}}}               %
\newcommand{\iid}{i.i.d~}                                                        %
\newcommand{\Grad}{\nabla}                                                       %
\DeclarePairedDelimiterX{\Inner}[2]{\langle}{\rangle}{#1, #2}                    %

\newcommand{\MI}{\mathnormal{I}}                                                     %

\newcommand{\Land}{\wedge}                                                       %
\newcommand{\Lor}{\vee}														     %
\newcommand{\what}[1]{\widehat{#1}}                                              %
\newcommand{\defeq}{\overset{\mathrm{def}}{=}}                                                      %

\DeclareMathOperator*{\union}{\cup}
\DeclareMathOperator*{\intersection}{\cap}

\DeclareMathOperator*{\argmax}{argmax}

\newtheorem{definition}{Definition}

\newtheorem{assumption}{Assumption}
\newtheorem{lemma}{Lemma}
\newtheorem{theorem}{Theorem}
\newtheorem{remark}{Remark}
\newtheorem{corollary}{Corollary}

\def\independenT#1#2{\mathrel{\rlap{$#1#2$}\mkern2mu{#1#2}}}
\DeclareMathOperator{\independent}{\protect\mathpalette{\protect\independenT}{\perp}}  %
\allowdisplaybreaks      

\newcommand{\Set}[1]{\{#1\}}
\newcommand{\Xs}{\mathsf{X}}                   
\newcommand{\Vs}{\mathsf{V}}                   
\newcommand{\As}{\mathsf{A}}                   
\newcommand{\Bs}{\mathsf{B}}                   
\newcommand{\Cs}{\mathsf{C}}                   
\newcommand{\Es}{\mathsf{E}}                   
\newcommand{\Gs}{\mathcal{G}}                  
\newcommand{\Gsm}{\mathcal{G}_{m}}             
\newcommand{\Gsmk}{\mathcal{G}_{m,k}}          
\newcommand{\Gsrm}{\widetilde{\mathcal{G}}_{m}}             
\newcommand{\Gsrmk}{\widetilde{\mathcal{G}}_{m,k}}          
\newcommand{\Lsm}{\mathcal{G}_{m}^{l}}             
\newcommand{\Lsmk}{\mathcal{G}_{m,k}^{l}}          
\newcommand{\xv}{\vect{x}}                      
\newcommand{\Xc}{\mathcal{X}}                   
\newcommand{\Dc}{\mathcal{D}}                   
\newcommand{\Pf}{\mathcal{P}}                   
\newcommand{\Qf}{\mathcal{Q}}                   
\newcommand{\Qfz}{\mathcal{Q}_0}                   
\newcommand{\pii}{\mathsf{\pi}_i}               
\newcommand{\Pms}{\bm{\varphi}}                 
\newcommand{\Data}{\mathsf{S}}                  
\newcommand{\Dec}{\zeta}                       
\newcommand{\Est}[1]{\widehat{#1}}              
\newcommand{\err}{p_{\mathrm{err}}}             
\newcommand{\Par}[2]{\mathsf{\pi}_{#2}(#1)}  
\newcommand{\etai}{\vectsym{\eta}_{i}}       
\newcommand{\etaz}{\vectsym{\eta}_{0}}       
\newcommand{\Xpii}{\Xs_{\pii}}

\newcommand{\ST}{\vect{T}}                       
\newcommand{\EST}{\vect{\mathcal{T}}}            
\newcommand{\thetamin}{\theta_{\mathrm{min}}}

\newcommand{\eigmax}{\lambda_{\mathrm{max}}}

\newcommand{\wv}{\vect{w}}
\newcommand{\wbv}{\overline{\vect{w}}}

\newcommand{\wmax}{w_{\mathrm{max}}}

\newcommand{\Lf}{\mathcal{L}}                   
\newcommand{\Dl}{\Delta}
\newcommand{\etav}{\vectsym{\eta}}
\newcommand{\klmax}{\Delta_{\mathrm{max}}}
\newcommand{\sigmamin}{\sigma_{\mathrm{min}}}
\newcommand{\sigmamax}{\sigma_{\mathrm{max}}}
\newcommand{\mumax}{\mu_{\mathrm{max}}}

\newcommand{\Ball}{\mathbb{B}}

\begin{document}

\title{Information-theoretic limits of Bayesian network structure learning}

\author{Asish Ghoshal and Jean Honorio\\
Department of Computer Science\\
Purdue University\\
West Lafayette, IN - 47906\\
\{aghoshal, jhonorio\}@purdue.edu}

\date{}

\maketitle

\begin{abstract}
In this paper, we study the information-theoretic limits of learning the
structure of Bayesian networks (BNs), on discrete as well as
continuous random variables, from a finite number of samples. 
We show that the minimum number of samples required by \emph{any procedure} to recover
the correct structure grows as $\BigOm{m}$ and $\BigOm{k \log m + \nicefrac{k^2}{m}}$
for non-sparse and sparse BNs respectively, where $m$ is the number of variables and $k$
is the maximum number of parents per node.
We provide a simple recipe, based on an extension of the Fano's inequality,
to obtain information-theoretic limits of structure recovery
for any exponential family BN. We instantiate our result for
specific conditional distributions in the exponential family to characterize
the fundamental limits of learning various commonly used BNs,
such as conditional probability table based networks, Gaussian BNs,
noisy-OR networks, and logistic regression networks. 
En route to obtaining our main results, 
we obtain tight bounds on the number of sparse and non-sparse essential-DAGs. 
Finally, as a byproduct, we recover the information-theoretic limits
of sparse variable selection for logistic regression.
\end{abstract}

\section{Introduction}
\paragraph{Motivation.} Bayesian Networks (BNs) are a class of probabilistic graphical models that describe the conditional dependencies between a set of random variables as a directed acyclic graph (DAG).
However, in many problems of practical interest, the structure of the network is not known a priori and must be inferred from data. 

Although, many algorithms have been developed over the years for learning BNs (cf. \cite{koller2009probabilistic} and \cite{de2011efficient}
for a detailed survey of algorithms), an important question that has hitherto remained unanswered is the fundamental limits of
learning BNs, i.e., ``What is the minimum number of samples required 
by any procedure to recover the true DAG structure
of a BN?''. The answer to this question
would help shed light on the fundamental limits of learning the
DAG structure of BNs, and also help determine if existing algorithms
are optimal in terms of their sample complexity or if there exists a gap between
the state-of-the-art estimation procedures and the information-theoretic limits.
In this paper we obtain lower bounds on the minimum number of samples required to learn BNs
over $m$ variables, and sparse BNs over $m$ variables with maximum in-degree of $k$.
 
\paragraph{Contribution.} In this paper, we make the following contributions.
 We derive necessary conditions on the sample complexity of recovering the DAG structure of non-sparse
	and sparse BNs. We show that $\BigOm{m}$ samples are necessary
	for consistent recovery of the DAG structure of BNs, while for sparse networks
	$\BigOm{k \log m + \nicefrac{k^2}{m}}$ samples are necessary.
We provide a simple recipe for obtaining the information-theoretic limits of learning any exponential family
	BN, and we instantiate our result for specific conditional distributions to determine the
	fundamental limits of learning the structure of various widely used BNs, namely,
	conditional probability table (CPT) based networks, Gaussian networks, noisy-OR networks,
	and logistic regression networks. Our lower bound of  $\BigOm{k^2 \log m}$ 
	matches the upper bound on $\BigO{k^2 \log m}$, obtained by Ravikumar et al. \cite{Ravikumar2010}
	for $\ell_1$-regularized logistic regression. We also show that the 
	\emph{SparsityBoost} algorithm developed by Brenner and Sontag \cite{Brenner} for
    learning binary CPT BNs, which has a sample complexity of $\BigO{m^2(\nicefrac{1}{\thetamin})}$,
	is far from the information-theoretic limit of $\BigOm{\frac{k \log m + \nicefrac{k^2}{m}}{\log(\nicefrac{1}{\thetamin})}}$,
	where $\thetamin$, is the minimum probability value in the conditional probability tables.
	An interesting corollary of our main result is that learning \emph{layered BNs} ---
	where the ordering of the nodes is known (upto layers) and the parent set of each variable is constrained to be in the layer
	above it --- is as hard as learning general BNs in terms of their sample complexity.
Lastly, of independent interest, are our extension of the Fano's inequality to the case where there are latent variables,
and our upper-bound on the KL divergence between two exponential family distributions
as the inner product of the difference between the natural parameter and the expected sufficient statistics. 

\section{Related Work}
\label{sec:related_work}
H\"{o}ffgen \cite{hoffgen1993learning}, and Friedman and Yakhini \cite{friedman1996sample}
were among the first to derive sample complexity results for learning BNs.
In both \cite{hoffgen1993learning} and \cite{friedman1996sample}
the authors provide upper bounds on sample complexity of learning
a BN that is likelihood consistent, i.e., the likelihood of the learned network 
is $\epsilon$ away from the likelihood of the true network
in terms of the Kullback-Leibler (KL) divergence measure.
Abbeel  et al. \cite{Abbeel05} provide polynomial sample complexity results for learning likelihood
consistent factor graphs.

Among sample complexity results for learning structure consistent BNs,
where the structure of the learned network is close to the true network,
Spirtes et al. \cite{spirtes2000causation} and Cheng et al. \cite{Cheng02} provide such guarantees for
polynomial-time test-based methods, but the results hold only in the infinite-sample limit.
Chickering and Meek in \cite{Chickering02b} also provide a greedy hill-climbing
algorithm for structure learning that is structure consistent in the infinite sample limit.
Zuk et al. \cite{Zuk06} show structure consistency of a single network 
and do not provide uniform consistency for all candidate networks, i.e., 
the bounds relate to the error of learning a specific wrong network having a score
greater than the true network. 
Brenner and Sontag \cite{Brenner} provide upper bounds on the sample complexity of 
recovering the structure of sparse BNs. However, they consider
binary valued variables only and the sample complexity grows as $\BigO{m^2}$.

Fano's method has also been used to obtain lower bounds on the sample complexity
of undirected graphical model (Markov random fields or MRFs) selection. 
See Appendix \ref{app:comp_with_mrfs} for results for MRFs and technical
differences between learning BNs and MRFs.

\section{Preliminaries}
\label{sec:prelim}
Let $\Xs = \Set{X_1, \ldots, X_m}$ be a set of random variables,
where $X_i \in \Xc_i, \forall i \in [m]$. Let $\Dc \defeq \times_{i=1}^m \Xc_i$ be
the domain in which the variables in $\Xs$ jointly take their values.
A BN for $\Xs$ is a tuple $(G, \Pf(G, \Theta))$; where $G = (\Vs, \Es)$ is a directed acyclic graph (DAG)
with $\Vs = [m]$ being the vertex set and $\Es \subset [m] \times [m]$ being the set of directed edges,
and $\Pf(G, \Theta)$ is a probability distribution over $\Xs$ that is parameterized by $\Theta$ and
factorizes according to the DAG structure $G$. Particularly, $\forall \xv \in \Dc$,
$\Pf(\xv; G, \Theta) = \prod_{i = 1}^m \Pf_i(x_i; \pii(G), \Theta)$, 
where $\pii(G) \subseteq [m]\setminus\Set{i}$ is the parent set of the $i$-th node in $G$,
$\Xs_{\pii(G)} = \Set{X_j | j \in \pii(G)}$, $\Pf_i(.) = \Prob{x_i | \Xs_{\pii(G)}, \Theta_i}$ is the conditional distribution of $X_i$
given an assignment to its parent set, and $\Theta_i$ are the parameters for the $i$-th conditional distribution.

The DAG structure $G$ of a BN specifies the conditional independence relationships that exist between
different random variables in the set $\Xs$. Different graph structures which make the same conditional
independence assertions about a set of random variables are called Markov equivalent. 
\begin{definition}[Markov equivalence]
Two DAGs $G_1 = (\Vs, \Es_1)$ and $G_2 = (\Vs, \Es_2)$ are Markov equivalent
if for all disjoint subsets $\As, \Bs, \Cs \subset \Vs$, $\Xs_{\As} \independent \Xs_{\Bs} | \Xs_{\Cs}$ in $G_1$
$\iff$ $\Xs_{\As} \independent \Xs_{\Bs} | \Xs_{\Cs}$ in $G_2$.
\end{definition}
The set of DAGs that are Markov equivalent to the DAG $G$ is denoted by $[G]$. 
An essential graph\footnote{See Andersson et. al. \cite{andersson1997markov} for a formal definition.}, 
consisting of both directed edges, which are called protected edges, and undirected edges, is a canonical representation
of the (Markov) equivalence class of DAGs. The undirected edges can be oriented in either direction
without changing the conditional independence relationships encoded by the graphs. We denote by $G^*$
the essential graph for $[G]$.

\section{Problem Formulation}
\label{sec:problem}
Let $\Gs$ be an ensemble of DAGs. We denote by $\Gsm$ the ensemble of DAGs over $m$ nodes.
Also, let $\Pms(\Gs)$ be some set of ``parameter maps''.
A parameter map $\Theta \in \Pms(\Gs)$, maps a given DAG structure $G$ to a specific instance of the conditional
distribution parameters that are compatible with the DAG structure $G$, i.e., $\Theta(G)$. 
It is useful to think of $\Theta$ 
as a policy for setting the parameters of the conditional distributions, given a DAG $G$.
For instance, for binary CPT networks, a particular policy $\Theta$ would consist of
several candidate probability tables for each node, one for each possible number of parents the node can have 
(from $0$ to $m-1$), with entries set to some specific values. 
Then, given a DAG structure G, $\Theta(G)$ assigns a probability table
to each node (from the policy $\Theta$) according to the number of parents of the node in G.
This notion of parameter maps affords us the ability to generate a BN by sampling
the DAG structure and the parameters independently of each other, which, as would be evident later, is
a key technical simplification.
Let $\Pf_{\Gs}$ and $\Pf_{\Pms(\Gs)}$ be probability measures on the set $\Gs$
and $\Pms(\Gs)$ respectively. Nature picks a graph structure $G$, according to $\Pf_{\Gs}$,
and then samples a parameter map $\Theta$, independently, according to $\Pf_{\Pms(\Gs)}$.
Thereafter, nature generates a data set of $n$ \iid observations, $\Data = \Set{\xv^{(i)}}_{i = 1}^n$,
with $\xv^{(i)} \in \Dc$, from the BN $(G, \Pf(G, \Theta(G)))$. 
The problem of structure learning in BNs concerns with estimating
the graph structure $\what{G}$, up to Markov equivalence, from the data set $\Data$. In that context, we
define the notion of a decoder. A decoder is any function $\Dec: \Dc^n \rightarrow \Gs$
that maps a data set of $n$ observations to an estimated DAG $\Est{G} \in \Gs$. The estimation error
is defined as follows
\begin{align}
\err \defeq \inf_{\Dec} \sup_{\Pf_{\Pms(\Gs)}} \sum_{G} \int_{\Theta \in \Pms(\Gs)} \Big( \Pf_{\Pms(\Gs)}(\Theta) \Pf_{\Gs}(G) 
	 \Prob{\Dec(\Data) \notin [G] | G, \Theta} \Big), \label{eq:err}
\end{align}
where the probability $\Prob{.| G, \Theta}$ is computed over the data distribution $\Pf(G, \Pf(G, \Theta(G)))$
for a specific DAG structure $G$ and parameters $\Theta(G)$.

\emph{Note that our definition of estimation error is stronger than what is typically used in the literature
for structure recovery of MRFs (see e.g. \cite{santhanam2012information} and \cite{Wang2010}),
since we focus on the maximum error across all measures $\Pf_{\Pms(\Gs)}$ over the parameter maps $\Pms(\Gs)$
which itself can be uncountable.}
Here, we are interested in obtaining necessary conditions for consistent structure recovery of
BNs, i.e., we show that if the number of samples is less than a certain threshold, 
then any decoder $\Dec$ fails to recover the true graph structure with probability of error $\err > \nicefrac{1}{2}$.

\emph{We emphasize that while our sample complexity results invariably depend on the 
parameter space $\Pms(\Gs)$ under consideration, the decoder only has access to the data set $\Data$.}
Apart from $\Gsm$, we consider various other ensembles of DAGs in this paper, to fully characterize the
fundamental complexity of learning different classes of BNs. Among the ensembles we consider,
$\Gsmk$ denotes the family of DAGs, where each node is allowed to have
at most $k$ parents. We also consider generalizations of QMR-DT \cite{shwe1991probabilistic} 
type two-layered BNs, to multiple layers of nodes, with nodes in each layer
only allowed to have parents in the layer above it. 

Let $\mathcal{V} = \Set{\Vs_i}_{i=1}^l$,
define an ordering of $m$ nodes into $l$ layers where $\Vs_i$ is the set of nodes in the $i$-th layer.
We have that $\Abs{\Vs_i} = m_i$ and $\sum_{i=1}^l \Abs{\Vs_i} = m$.
$\Lsm(\mathcal{V})$ denotes an ensemble of DAG structures where $\forall G = (\Vs, \Es) \in \Lsm$, 
$\Vs = \bigcup_{\Vs_i \in \mathcal{V}} \Vs_i$ and $\Es = \Set{(u, v) | u \in \Vs_{i + 1} \Land v \in \Vs_{i},\; i \in  [l - 1]}$.
We write $\Lsm$ instead of $\Lsm(\mathcal{V})$ to indicate that the members of $\Lsm$ have some \emph{known}
layer-wise ordering of the $m$ nodes, without making the ordering explicit.
Finally, we consider another ensemble $\Lsmk \subset \Lsm$
where the nodes are allowed to have at most $k$ parents. Together, the ensembles
$\Gsm, \Gsmk, \Lsm$ and $\Lsmk$, span a wide range of the sample complexity landscape of
recovering the structure of BNs. 
In the following section we present our main result on the fundamental limits of
learning BNs. 

\section{Main Results}
\label{sec:main_results}
Fano's inequality is one of the primary tools used for deriving
necessary conditions on structure recovery of graphical models. The difficulty of recovering
the DAG structure of a BN, however, depends both on the structural properties of the ensemble
of DAG structures under consideration, as well as on the conditional distributions and their parameters.
In order to obtain guarantees about structure recovery, we treat the parameters of the conditional
distributions as latent variables --- variables that we do not observe and are not interested in estimating.
Given that the likelihood of the observed data depends, both on the structure and parameters of
the BN that generated the data, it behooves us to ask: ``\emph{If we are only
interested in recovering the structure of BNs, do the presence of
unobserved parameters make structure estimation easier or harder?}''
To rigorously answer this question, we extend the classic Fano's
inequality, which is defined for a Markov chain, to a slightly more general setting as given below.
\begin{figure}[h]
	\begin{minipage}{0.5\linewidth}
	\centering
	\includegraphics[width=0.4\linewidth]{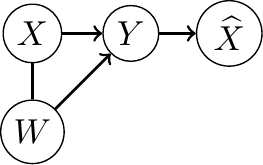}\\
	\centering
	(a)
	\end{minipage}%
	\begin{minipage}{0.5\linewidth}
	\centering
	\includegraphics[width=0.4\linewidth]{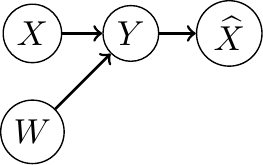}\\
	\centering
	(b)
	\end{minipage}
	\caption{\small Fano's inequality extension. In (a) the edge between $W$ and $X$
	is undirected to indicate that the edge can be oriented in either direction. \label{fig:fano}%
	}
\end{figure}
\begin{theorem}[Fano's inequality extension]
\label{thm:fano}
Let $W, X,$ and $Y$ be random variables and let $\what{X}$ be any estimator of $X$.
If the random variables are related according to the graphical model in Figure \ref{fig:fano} (a), then
\begin{align}
	\Prob{X \neq \what{X}} \geq 1 - \frac{\MI(Y ; X | W) + \log 2}{H(X|W)}. \label{eq:thm_fano1}
\end{align}
Moreover, if $W \in \mathcal{W}$, is independent of $X \in \mathcal{X}$ (Figure \ref{fig:fano} (b)),
and $\Pf_{\mathcal{X}}$ and $\Pf_{\mathcal{W}}$ be any probability measures over
 $\mathcal{X}$ and $\mathcal{W}$ respectively, then,
\begin{align}
&\sup_{\Pf_{\mathcal{W}}}\sum_{x \in \mathcal{X}} \int_{w \in \mathcal{W}}
	 \kern-1em \Prob{x \neq \what{X} |X=x , W = w} \Pf_{\mathcal{W}}(w) \Pf_{\mathcal{X}}(x)  \notag \\
	  & \qquad \geq  1 - \frac{ \sup_{w \in \mathcal{W}} \MI(Y ; X | W = w) + \log 2}{H(X)}. \label{eq:thm_fano2}
\end{align}
\end{theorem}

Proofs of main results can be found in Appendix \ref{app:proofs_main_results}.
\begin{remark}
Theorem \ref{thm:fano} can be seen an extension of Fano's inequality to the case
where there are latent variables $W$ that influence $Y$, while we are interested in
only estimating $X$. If in Figure \ref{fig:fano} we have that $W \rightarrow X$,
then $\MI(X; Y | W) \geq \MI(X; Y)$. Further, since $H(X | W) \leq H(X)$, we conclude
that presence of the latent variable $W$, reduces the estimation error $\err$.
\end{remark}
\begin{remark}
When $X$ and $W$ are independent, we get $\MI(X; Y | W) = \MI(X; Y)$,
and \eqref{eq:thm_fano1} reduces to the well known Fano's inequality. However, the conditional
mutual information $\MI(X; Y | W = w)$ in \eqref{eq:thm_fano2} can be computed easily as compared to $\MI(X; Y)$ 
when we have access to the conditional distribution of $Y$ given $X$ and $W$.
Also note that we do not need $\mathcal{W}$ to be countable.
\end{remark}
\begin{figure}
	\centering
	\includegraphics[width=0.3\linewidth]{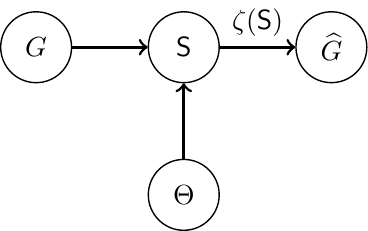}
	\caption{\small The DAG structure $G$ and a parameter map $\Theta$ are sampled independently. 
	Data set $\Data$ of $m$ samples is generated from the BN $(G, \Pf(G, \Theta(G))$.
	A decoder $\Dec$ then estimates the DAG structure $\what{G} = \Dec(\Data)$.\label{fig:schematic}}
\end{figure}
Theorem \ref{thm:fano} serves as our main tool for lower bounding the estimation error,
and subsequently obtaining necessary conditions on the number of samples. In order to obtain
sharp lower bounds on $\err$, we assume that
DAG structures and parameters maps (and by extension parameters) are sampled independently.
Figure \ref{fig:schematic} shows the schematics of the inference procedure.

\emph{Henceforth, we will use the terms ``parameter maps'' and ''parameters'' interchangeably, since
given a DAG structure $G$, the parameter map $\Theta$, maps $G$ to a specific parameterization
of the BN.}
Then, given any ensemble of DAG structures $\mathcal{G}$, the main steps involved in using Theorem \ref{thm:fano}
to lower bound $\err$ are: (a) obtaining lower bounds on $H(G) = \log \Abs{\mathcal{G}}$, 
which follows from our assumption that $G$ is sampled uniformly from $\mathcal{G}$, and
(b) computing the mutual information between the data set and DAG structures over
all possible parameter choices, i.e., $\sup_{\Theta \in \Pms(\mathcal{G})} \MI(\Data; G | \Theta)$.
To accomplish each of the above objectives, we consider restricted
ensembles $\Gsrm \subset \Gsm$ and $\Gsrmk \subset \Gsmk$, of size-one equivalence classes,
i.e., $\forall G \in \Gsrm \Lor G \in \Gsrmk,\, \Abs{[G]} = 1$. Note that for
any graph $G$ in $\Lsm$ or $\Lsmk$, we have that $\Abs{[G]} = 1$, since edges are constrained
to go from layer $(i+1)$ to $i$.  Thus, the ensembles $\Gsrm, \Gsrmk, \Lsm$ and $\Lsmk$
can be thought of as consisting of essential DAGs, where all edges are protected. 
In the following section, we bound the number of essential DAGs in each of the restricted ensembles.
\paragraph{Enumerating DAGs.}
Essential DAGs, i.e., Markov equivalent classes of DAGs of size 1,
was first enumerated by Steinsky \cite{Steinsky2003}. However, the number of DAGs is given as a
recurrence relation, for which a closed form solution is difficult to compute. Therefore, we
compute tight bounds on the number of essential DAGs in the following paragraphs. 
In the following lemmas we bound the number of DAGs in each of the
restricted ensembles introduced previously. 
\begin{lemma}
\label{lemma:gsrm_size}
The size of the restricted ensemble $\Gsrm$ is bounded as follows:
\begin{align}
2^{(\nicefrac{m(m - 3)}{2}) + 1} \leq \Abs{\Gsrm} \leq m!\, 2^{\nicefrac{m(m - 1)}{2}},
\end{align}
and $\log \Abs{\Gsrm} \geq ((\nicefrac{m(m - 3)}{2}) + 1) \log 2$.
\end{lemma}%
Note that the lower bound in Lemma \ref{lemma:gsrm_size} is asymptotically tight.
Now we bound the number of essential DAGs where each node is allowed to have at most $k$ parents. 

\begin{lemma}
\label{lemma:gsrmk_size}
Assuming $k > 1$ and $m > 2$, the size of the restricted ensemble $\Gsrmk$ is bounded as follows:
\begin{gather}
2^{(\nicefrac{k(k - 3)}{2}) + 1} \prod_{j=k+1}^{m-1} \left(\sum_{i=0}^k {j - 1 \choose i} \right) \leq \Abs{\Gsrmk} 
 \leq m!\, 2^{\nicefrac{k(k - 1)}{2}} \prod_{j=k+1}^{m-1} \left(\sum_{i=0}^k {j \choose i} \right),
\end{gather}
and $\log \Abs{\Gsrmk} \geq k\big\{\log (m-2)! - (m-k-2) \log k -\log k! \big\} 
	  + \{(\nicefrac{k(k - 3)}{2}) + 1\} \log 2$.
\end{lemma}
Note that using Stirling's factorial formula, the above lemma gives the following lower bound
on the number of sparse essential DAGs: $\log \Abs{\Gsrmk} = \BigOm{k m \log m}$. 
Further, a little calculation shows that $\log \Abs{\Gsrmk} = \BigO{k m \log m}$ for large enough $m$.
Thus our bounds for the number of sparse essential DAGs is tight.
The following lemma bounds the number of ``layered'' essential DAGs.
\begin{lemma}
\label{lemma:ls_size}
The number of BNs in the family $\Lsm$ and $\Lsmk$ is as follows:
\begin{align}
\Abs{\Lsm} = \prod_{i=1}^{l - 1} (2^{m_{i + 1}})^{m_i},\;
\Abs{\Lsmk} = \prod_{i=1}^{l - 1} \left[\sum_{j=0}^k {m_{i+1}\choose j} \right]^{m_i}.
\end{align}
Further, $\log \Abs{\Lsm}$ and $\log \Abs{\Lsmk}$ are given as follows:
\begin{gather*}
\log \Abs{\Lsm} = (\log 2) \sum_{i=1}^{l - 1} (m_{i + 1})(m_i),\;
\log \Abs{\Lsmk} \geq k \sum_{l=1}^{l - 1} m_i \log \left(\frac{m_{i+1}}{k}\right).
\end{gather*}
\end{lemma}%

Next, we compute bounds on the mutual information between the data set and DAG structures.
\paragraph{Mutual Information Bounds.}
The mutual information $\MI(\Data; G | \Theta)$
cannot be computed exactly, in general. Therefore, we use the following lemma to bound the mutual
information from above.
\begin{lemma}
\label{lemma:mi_bound_general}
Let $\Pf_{\Data | G, \Theta}$ be the distribution
of $\Data$ conditioned on a specific DAG $G$ and specific parameters $\Theta$,
and let $\Qf$ be any distribution over $\Data$. Then we have
\begin{align}
\sup_{\Theta \in \Pms(\mathcal{G})} \MI(\Data; G | \Theta) 
	&\leq \sup_{\Theta \in \Pms(\mathcal{G})} \frac{1}{\Abs{\mathcal{G}}} \sum_{G \in \mathcal{G}} 
		\KL{\Pf_{\Data | G, \Theta}}{\Qf}.	\label{eq:mi_kl}
\end{align}
\end{lemma}
Assuming that $\Xc_i = \Xc,\, \forall i \in [m]$, we chose
the distribution $\Qf$ to be the product distribution $\Qf = \Qfz^{m n}$,
where $\Qfz = \Pf_i(\varnothing, \Theta_i)$. In other words, the distribution $\Qf$ is chosen
to be the distribution encoded by a DAG with no edges.

The main hurdle in using \eqref{eq:mi_kl} to bound the mutual information $\MI(\Data, G | \Theta)$,
is computing the KL divergence $\KL{\Pf_{\Data | G, \Theta}}{\Qf}$. \emph{Often times, in BNs
characterized by local conditional distributions, coming up with a closed form solution
for the joint distribution over all nodes or even the marginal distribution of an arbitrary node is not possible};
unless we assume that the marginal distribution of the parents of a node form a \emph{conjugate prior} for the
conditional distribution of the node --- an assumption which is quite restrictive. Therefore, to tackle the above
problem, we derive the following upper bound on the KL divergence for exponential family distributions
which is easy to compute.
\begin{lemma}[KL Divergence Bound for Exponential Family Distributions]
\label{lemma:kl_bound_exp}
Let $X \in \R^d$ be any random variable. Let $\Pf_1$ and $\Pf_2$ be distributions over $X$,
belonging to the exponential family, with natural parameters $\etav_1$ and $\etav_2$ respectively,
i.e. 
\begin{align*}
\Pf_1(x) = \exp(\etav_1^T \ST(x) - \psi(\etav_1)) h(x),
\end{align*}
where $\ST(X)$ is the sufficient statistics
(similarly for $\Pf_2$). Assuming $h(x) \neq 0~ \forall x \in \R^d$,  we have
\begin{gather}
\KL{\Pf_1}{\Pf_2} \leq \Dl(\etav_1, \etav_2), \\
\Dl(\etav_1, \etav_2) \defeq (\etav_1 - \etav_2)^T (\EST(\etav_1) - \EST(\etav_2)),
\end{gather}
where $\EST(\etav_1) \defeq \Exp{X}{\ST(x) | \etav_1}$ is the expected sufficient statistic
of $X$ as computed by the distribution parameterized by $\etav_1$ (similarly for $\EST(\etav_2)$). 
\end{lemma}
Note that even though $\KL{\Pf_1}{\Pf_2}$ is not symmetric, its upper bound $\Dl(\etav_1, \etav_2)$ is symmetric.
Given the fact that $\Data$ is sampled \iid from $\Pf(G, \Theta(G))$,
which in turn factorizes as a product of conditional distributions $\Pf_i$, 
we then have the following result for the mutual information $\MI(\Data; G | \Theta)$.
\begin{lemma}[Mutual Information Bound]
\label{lemma:mi_bound}
For any ensemble of DAG structures $\mathcal{G}$, we have
\begin{gather}
\sup_{\Theta \in \Pms(\mathcal{G})} \MI(\Data; G | \Theta) \leq 
	\frac{n}{\Abs{\mathcal{G}}} \sum_{G \in \mathcal{G}} \sum_{i=1}^m \:\: \sup_{\mathclap{\Theta \in \Pms(\mathcal{G})}}
		 \Exp{\Xpii}{\KL{\Pf_i(\pii(G), \Theta)}{\Qfz}},
\end{gather}
where $\Pf_i(\pii(G), \Theta)$ is the conditional distribution of the $i$-th node.
Further, if we have that, $\forall i \in [m], X_i \in \Xc$ and 
$\Pf_i(\pii(G), \Theta_i)$ belongs to the exponential family
with natural parameter $\etai \defeq \vectsym{\eta}(\Xpii, \Theta_i)$
and $\Qfz$ belongs to the exponential family with natural parameter $\etaz$; then, 
\begin{align*}
\sup_{\mathclap{\Theta \in \Pms(\mathcal{G})}} \MI(\Data; G | \Theta) \leq 
	\frac{n}{\Abs{\mathcal{G}}} \sum_{G \in \mathcal{G}} \sum_{i=1}^m 
		\:\: \sup_{\mathclap{\Theta \in \Pms(\mathcal{G})}} \Exp{\Xpii}{\Dl(\etai, \etaz)}.
\end{align*}
\end{lemma}
\begin{remark}
In the above lemma, $\Dl(\etai, \etaz)$ is a random variable because the natural parameter $\etai$ depends
on the parents $\Xpii$. The quantity $\Dl(\etai, \etaz)$ in the above lemma is non-negative and
measures  how far the  conditional distribution of a variable with parents $\pii(G)$ is from the 
distribution of the variable with no parents, as a function of the difference
between the expected sufficient statistics and the natural parameters. The mutual information between 
the data set $S$ and the DAG structure $G$ is then a sum of the expected ``distances'' of the conditional
distributions from the distribution of a variable with no parents.
\end{remark}
With the exception of Gaussian BNs,
where we can write the joint and marginal distributions of the variables
in closed form, it is in general difficult to compute the expectation of $\Dl(\etai, \etaz)$.
Therefore, we bound the mutual information by bounding $\Dl(\etai, \etaz)$, which
can be easily done for bounded random variables.
From the above lemma, we then get the following mutual information bound for layered BNs.
\begin{corollary}[Mutual Information Bound for Layered BNs]
\label{cor:mi_layered}
If $\Gs = \Lsm(\mathcal{V}) \Lor \Gs = \Lsmk(\mathcal{V})$, then
\begin{gather*}
\sup_{\mathclap{\Theta \in \Pms(\mathcal{G})}} \MI(\Data; G | \Theta) \leq  
	\frac{(m - m_l)n}{\Abs{\mathcal{G}}} \sum_{G \in \mathcal{G}} \Bigl\{
		\max_{i \in \Vs \setminus \Vs_l} \sup_{\Theta \in \Pms(\mathcal{G})} 
		 \Exp{\Xpii}{\KL{\Pf_i(\pii(G), \Theta)}{\Qfz}} \Bigr\},
\end{gather*}
where we recall that $\mathcal{V} = \{\Vs_j\}_{j=1}^l$ is an ordering of nodes into $l$ layers, $\Vs_j$ is the set of
nodes in the $j$-th layer and $m_j = \Abs{\Vs_j}$. Further, for exponential family conditional distributions, we have
\begin{gather*}
\sup_{\mathclap{\Theta \in \Pms(\mathcal{G})}} \MI(\Data; G | \Theta) \leq  
	\frac{(m - m_l)n}{\Abs{\mathcal{G}}} \sum_{G \in \mathcal{G}} \Bigl\{
		\max_{i \in \Vs \setminus \Vs_l} \sup_{\Theta \in \Pms(\mathcal{G})} 
		  \Exp{\Xpii}{\Dl(\etai, \etaz)} \Bigr\}.
\end{gather*}
\end{corollary}

In order to obtain tight sample complexity results,
we need to create ``difficult instances'' of BNs that are hard to learn.
Intuitively speaking, inferring the parents of a node will be hard if the conditional distribution
of a node with many parents is close to that of a node with no parents.
Therefore, we make the following crucial assumption about the conditional distributions
specified by the BN $(G, \Pf(G, \Theta(G)))$.
\begin{assumption}
\label{ass:kl}
The KL divergence between the conditional distributions $\Pf_i(\pii(G), \Theta)$
and $\Qfz \defeq \Pf_i(\varnothing, \Theta)$ over the variable $X_i$ is bounded by a constant,
which for exponential family distribution translates to:
\begin{align*}
\sup_{\Theta \in \Pms(\Gs)} \Exp{\Xpii}{\Dl(\etai, \etaz)} \leq \klmax, \forall i \in [m],
\end{align*}
where $\klmax > 0$ is a constant.
\end{assumption}
\emph{We show that under certain parameterizations, Assumption \ref{ass:kl} holds for many commonly
used BNs under very mild restrictions on the parameter space $\Pms(\Gs)$.}
\begin{theorem}
\label{thm:main}
Let $\Xs = \Set{X_1, \ldots, X_m}$ be a set of random variables with $X_i \in \Xc$.
Let $G$ be a DAG structure over $\Xs$ drawn uniformly at random from some family $\Gs$ of DAG structures. 
Further assume that we are given a  data set $\Data$ of $n$ \iid samples drawn 
from a BN $(G, \Pf(G, \Theta(G)))$, where the parameter map $\Theta$, 
is drawn from some family $\Pms(\Gs)$. Assuming that the condition in Assumption $\ref{ass:kl}$ 
holds for the conditional distributions $\Pf_i$, then If $n \leq \Lf(\Gs)$, than any decoder $\Dec$ fails to 
recover the true DAG structure $G$ with probability $\err \geq 1/2$, where $\err$ is
defined according to \eqref{eq:err}.
The necessary number of samples $\Lf(\Gs)$, for various classes of BNs, is given as follows:
\begin{align*}
\Lf(\Gsrm) &= \frac{\log 2}{\klmax}\left(\frac{m-3}{2} - \frac{1}{m}\right), \\
\Lf(\Gsrmk) &= \frac{1}{2 m \klmax}\Big(k\Big\{\log (m-2)! -\log k!  - (m-k-2) \log k \Big\} \\
	&\qquad + \{(\nicefrac{k(k - 3)}{2}) - 1\} \log 2 \Big) , \\
\Lf(\Lsm) &= \frac{\log 2}{2(m - m_l)\klmax}\left(\sum_{i=1}^{l-1} m_{i + 1} m_i - 2 \right), \\
\Lf(\Lsmk) &= \frac{ k\sum_{i=1}^{l-1} m_i \log \Bigl(\frac{m_{i+1}}{k}\Bigr) - 2\log 2 }{2(m - m_l)\klmax}.
\end{align*}
\end{theorem}
Since the difficulty of learning a class of BNs is determined by the
difficulty of learning the most difficult subset within that class, we 
immediately get the following corollary on the fundamental limits of 
learning non-sparse and sparse BNs.
\begin{corollary}[Fundamental limits of BN structure learning]
\label{cor:main}
The necessary number of samples required to learn BNs on $m$ variables
and sparse BNs on $m$ variables and maximum number of parents $k$ is as follows:
\begin{gather*}
\Lf(\Gsm) = \frac{\log 2}{\klmax}\left(\frac{m-3}{2} - \frac{1}{m}\right), \\
\Lf(\Gsmk) = \frac{1}{2 \klmax} \Big(k \log(m-2) + \frac{k(k-3)\log 2}{2m} -  R(m,k) \Big), 
\end{gather*}
where $R(m, k) = (\nicefrac{k}{m}) \Big\{(m-2) + 2\log(m-2) + \log k! 
	+ (m-k-2) \log k \Big\} + \nicefrac{(\log 2)}{m}$.
\end{corollary}
In the above corollary, the lower bounds for non-sparse and sparse BNs come from 
the results for the restricted ensembles $\Gsrm$ and $\Gsrmk$ respectively.
\begin{remark}
For sparse BNs, if $k$ is constant with respect to $m$, then the reminder term in the lower bound is $R(m, k) \leq 1$
for sufficiently large $m$ and the sample complexity grows as $\BigOm{\log m}$. 
In this regime, our lower bound for learning sparse Bayesian networks matches
the analogous lower bound of $\BigO{\log m}$ for learning bounded degree Ising models
as obtained by Santhanam and Wainwright \cite{santhanam2012information}.
In general, however, 
$R(m, k) = \BigO{k \log k}$; therefore the number of samples necessary to learn sparse BNs grows
as $\BigOm{k \log m + \nicefrac{k^2}{m}}$.
\end{remark}
\begin{remark}
Note that since the number of samples required to learn both general essential DAGs ($\Lf(\Gsrm)$) and layered
networks ($\Lf(\Lsm)$) grows as  $\BigOm{m}$. This, 
combined with the fact that our lower bounds on the number of DAG structures in the ensemble
$\Gsrm$ was tight, leads us to the conclusion that the ordering of variables does not add
much to the difficulty of learning BNs in terms of sample complexity.
\end{remark}
Theorem \ref{thm:main} provides a simple recipe for obtaining necessary conditions on the sample complexity
of learning any exponential family BN, as we demonstrate in the next section.

\begin{table}[htbp]
\centering
\renewcommand{\arraystretch}{1.7}
\setlength\tabcolsep{2pt}
\begin{tabular}{lcccc}
\hline
& Non-Sparse & Sparse \\
\hline
CPT & ${\frac{m}{\log(\nicefrac{1}{\thetamin})}}$ & ${\frac{k \log m + \nicefrac{k^2}{m}}{\log(\nicefrac{1}{\thetamin})}}$ \\
Gaussian & ${\frac{\sigmamin^2 m}{\sigmamin^2 + 2\mumax^2 (\wmax^2 + 1)}}$ & 
	${\frac{\sigmamin^2 (k \log m + \nicefrac{k^2}{m})}{\sigmamin^2 + 2(\wmax^2 + 1) \wmax^2}}$ \\
Noisy-OR & ${\frac{m}{\Abs{\log(\nicefrac{\hat{\theta}}{1 - \hat{\theta}})}}}$ &
	 ${\frac{ k \log m + \nicefrac{k^2}{m}}{\Abs{\log(\nicefrac{\hat{\theta}}{1 - \hat{\theta}})}}}$ \\
Logistic & ${\frac{m}{\wmax^1}}$ & ${\frac{k \log m + \nicefrac{k^2}{m}}{\wmax^1}}$ \\
\hline
\end{tabular}
\caption{\small Fundamental limits of learning the structure of various types of BNs
from Corollary \ref{cor:main}, where the entries of the tables are lower bounds, i.e., $\BigOm{.}$.
For CPT BNs, $\thetamin$ is the minimum entry in the conditional probability tables.
For Gaussian BNs, $\wmax, \mumax$ and $\sigmamin$ are the maximum $\ell_2$
norm of the weight vectors, maximum absolute mean and minimum conditional variance respectively.
For noisy-OR networks, $\hat{\theta} \in (0, 1)$ is the failure probability. Lastly,
for logistic regression, $\wmax^1$ is the maximum $\ell_1$ norm of the weight vectors.
 \label{tab:summary}}
\end{table}

\section{Implications for Commonly Used Bayesian Networks}
\label{sec:common_networks}
In this section, we instantiate Theorem \ref{thm:main}
for specific conditional distributions, to derive fundamental limits of learning various 
widely used BNs. This also allows us to highlight the role of parameters of the conditional
distributions in the sample complexity of learning the DAG structure of BNs.
Table \ref{tab:summary} summarizes our results for various commonly used BNs.
Proofs of results derived in this section can be found in Appendix \ref{app:proof_common_bns}.
\paragraph{Conditional Probability Table BNs.}
CPT BNs are perhaps the most widely used BNs, where the conditional distribution
of a node given its parents is described by probability tables.
As is typically the case, we assume that the support of $X_i \in \Xc = [v]$ for all $i$.
The conditional distribution of $X_i$ is given by the following
categorical distribution:
\begin{align*}
&\Pf_i(x_i; \pii(G), \Theta) = \prod_{j=1}^v \left(\theta_{ij}(\xv)\right)^{\Ind{X_i = j}},
\end{align*}
where $\Theta(G) \in \Pms(\mathcal{G})$ are the set of conditional probability tables for 
the variables $\Set{X_1, \ldots, X_m}$ compatible with the DAG structure $G$. 
Let us denote $\Theta^G \defeq \Theta(G)$. The conditional probability table
for the $i$-th random variable $\Theta_i^G : [v]^{\Abs{\pii}} \rightarrow \Delta_{v}$,
maps all possible assignments to the parent set $\Xpii$ to the $(v-1)$-dimensional probability simplex 
$\Delta_{v}$, and $\theta_{ij}^G(.)$ represents the $j$-th entry of the $v$-dimensional
vector $\Theta_i^G(.)$. The following lemma gives the upper bound on the mutual information
$\MI(\Data; G | \Theta)$ for CPT BNs. 
\begin{lemma}[Mutual Information bound for CPT networks]
\label{lemma:cpt_mi_bound}
For CPT BNs we have
\begin{align*}
\klmax &\leq 4\log(\nicefrac{1}{\thetamin}), \\
	\sup_{\Theta \in \Pms(\mathcal{G})} \MI(\Data; G | \Theta) &\leq 4 n m \log(\nicefrac{1}{\thetamin} ),	
\end{align*}
where $\thetamin > 0$ is minimum probability value in a probability table across all node 
and parent set assignments i.e.,
\begin{align*}
\thetamin &\defeq \inf_{\Theta \in \Pms(\mathcal{G})}  \min_{G \in \mathcal{G}}
	 \min_{i=1}^m \min_{\xv \in \Xc^{\Abs{\pii(G)}}} \min_{j=1}^v \theta_{ij}^G(\xv).
\end{align*}
\end{lemma}
\begin{sloppy}
\begin{remark}
The necessary number of samples required to learn dense and sparse CPT BNs is
$\BigOm{\frac{m}{\log(\nicefrac{1}{\thetamin})}}$ and 
$\BigOm{\frac{k \log m + \nicefrac{k^2}{m}}{\log(\nicefrac{1}{\thetamin})}}$, respectively. In the regime
that $\thetamin \geq \exp(-\nicefrac{1}{m})$, the sample complexity for learning dense and sparse
BNs is $\BigOm{m^2}$ and $\BigOm{km \log m + k^2}$, respectively.
\end{remark}
\end{sloppy}
\begin{remark}
The sample complexity of \emph{SparsityBoost} algorithm by Brenner and Sontag \cite{Brenner}
for recovering the structure of binary-valued, sparse, CPT BNs
grows as $\BigO{\max((\log m) \mu, m^2 \what{\mu}_P^2)}$, where $\mu$ is defined in \cite{Brenner} as
``the maximum inverse probability of an assignment to a separating set over all pairs of nodes'',
and is $1/\thetamin$ for the ensembles we consider.
The parameter $\what{\mu}_P^2$ is also defined as the maximum inverse probability of an assignment to 
a separating set but relates to the true graph, $G$, that generated the data and can be $\ll 1/\thetamin$.
If \emph{SparsityBoost} operates in the regime where the second term inside the $\max$ function dominates,
which the authors believe to be the case, then that leads to a sufficient condition of
$\BigO{(1/\thetamin)m^2}$, which is quite far from the information-theoretic limit.
\end{remark}

\paragraph{Gaussian BNs.}
In this case, we assume that the support $X_i \in \Xc = \R$ for all $i$,
the parameters of the $i$-th node $\Theta_i(G) = (\wv_i^G, \mu, \sigma^2)$,
and the conditional distributions are described by the following linear Gaussian model:
\begin{gather}
\Pf_i(\pii(G), \Theta) = \mathcal{N}(\mu_i, \nicefrac{\sigma^2}{2}), \label{eq:gaussian}\\
\mu_i = \left\{\begin{array}{cc}                     \label{eq:gaussian1}
	(\wv_i^G)^T \Xpii & \pii(G) \neq \varnothing,\\
	\mu & \text{otherwise}
\end{array} \right., 
\end{gather}
where $t_G: [m] \rightarrow [m]$ is a function that maps a node to its ``topological order''
as defined by the graph $G$. We assume that
\begin{gather*}
\wv_i^G \in \Ball^G_i \defeq \Set{\wv \in \R^{\Abs{\pii(G)}} | \NormII{\wv} \leq \nicefrac{1}{\sqrt{2(t_G(i) - 1)}}}, \\
\mu \in [\mu_a, \mu_b],\, \sigma \in [\sigmamin, \sigmamax], \text{ and }\\
\Theta(G) \in \left(\times_{i=1}^m \Ball_i^G\right) \times [\mu_a, \mu_b] \times [\sigmamin, \sigmamax],
\end{gather*}
where $-\infty < \mu_a \leq \mu_b < \infty$, $0 < \sigmamin \leq \sigmamax < \infty$. 
Accordingly, we have that $\Qfz = \Pf_i(\varnothing, \Theta) = \mathcal{N}(\mu, \frac{\sigma^2}{2})$,
and $\mu_i \leq \mu\NormII{\wv_i^G}$. 
Once again, we first bound the mutual information $\MI(\Data; G | \Theta)$ in the following lemma 
which we then plugin in Theorem \ref{thm:main}
to obtain the necessary conditions for learning Gaussian BNs. 
\begin{lemma}[Mutual Information bound for Gaussian networks]
\label{lemma:mi_bound_gaussian}
For Gaussian BNs we have: 
\begin{align*}
\klmax &\leq 1 + \frac{2\mumax^2(\wmax^2 + 1)}{\sigmamin^2} ,\\
\sup_{\Theta \in \Pms(\mathcal{G})} \MI(\Data; G | \Theta) &\leq nm\Biggl(1 + \frac{2\mumax^2(\wmax^2 + 1)}{\sigmamin^2}\Biggr),
\end{align*}
where $\mumax \defeq \max(\Abs{\mu_a}, \Abs{\mu_b})$ and $\wmax$ is the maximum $\ell_2$ norm of the weight vectors, i.e.
\begin{align*}
\wmax \defeq \sup_{\Theta \in \Pms(\Gs)} \max_{G \in \Gs} \max_{i=1}^m \NormII{\wv^G_i}.
\end{align*}
\end{lemma}

\begin{remark}
Invoking Theorem \ref{thm:main} for a two layer ordering of nodes, where there are $m - 1$ nodes in the top layer
and 1 node in the bottom layer, we recover the information-theoretic limits of linear regression.
Specifically, we have that the necessary number of samples for linear regression and sparse 
linear regression scale as
$\BigOm{\frac{\sigmamin^2 m}{\sigmamin^2 + 2\mumax^2 \wmax^2}}$ and 
$\BigOm{\frac{\sigmamin^2 k \log m}{\sigmamin^2 + 2\mumax^2 \wmax^2}}$
respectively.
\end{remark}
\begin{remark} The sample complexity of
learning the structure of degree bounded Gaussian MRFs scales as 
$\BigOm{\nicefrac{k \log(\nicefrac{m}{k})}{\log(1 + k \lambda)}}$ \cite{Wang2010}, where $\lambda$ is the
minimum correlation between pairs of nodes that are connected by an edge.
The corresponding result for sparse BNs is
$\BigOm{\frac{\sigmamin^2(k \log m + \nicefrac{k^2}{m})}{\sigmamin^2 + 2\mumax^2 \wmax^2}}$, which is 
slightly stronger than the corresponding lower bound for learning Gaussian MRFs, with respect to sparsity index $k$.
\end{remark}

\paragraph{Noisy-OR BNs.}
Noisy-OR BNs are another widely used class of BNs ---
a popular example being the two-layer QMR-DT network \cite{shwe1991probabilistic}.
They are usually parameterized by failure probabilities $\theta_{ij}$, which in 
the context of the QMR-DT network of diseases and symptoms can be interpreted
as the probability of not observing the $i$-th symptom given that the $j$-th disease is present.
More formally, we have binary valued random variables, i.e., $X_i \in \Xc = \Set{0,1}$ for all $i$,
the $i$-th conditional distribution is given by the Bernoulli distribution ($\mathcal{B}$)
 with parameter $\Theta_i = \theta \in (0, 1)$:
\begin{align}
\Pf_i(\pii(G), \Theta_i) = \mathcal{B}(1 - \theta_i) &&
\Pf_i(\varnothing, \Theta_i) = \mathcal{B}(\theta), \label{eq:noisy_or}
\end{align}
where $\theta_i = \theta \Big(\prod_{j \in \pii} \theta^{X_j} \Big)^{\nicefrac{1}{\Abs{\pii}}}$.
The following lemma bounds the mutual information for noisy-OR networks.
\begin{lemma}[Mutual Information bound for Noisy-OR]
\label{lemma:mi_bound_noisy_or}
For Noisy-OR BNs we have:
\begin{align*}
\klmax &\leq 2\Abs{\log(\nicefrac{\hat{\theta}}{(1 - \hat{\theta})})} \\
\sup_{\Theta \in \Pms(\mathcal{G})} \MI(\Data; G | \Theta) &\leq 2nm\Abs{\log(\nicefrac{\hat{\theta}}{(1 - \hat{\theta})})},
\end{align*}
where $\hat{\theta} \defeq \argmax_{\theta \in \Pms(\Gs)} \Abs{\log(\nicefrac{\theta}{(1 - \theta)})}$.
\end{lemma}%
\begin{remark}
From the above lemma we notice that recovering the structure of noisy-OR networks becomes
more difficult as the failure probability $\theta$ moves farther away from $\nicefrac{1}{2}$.
That is because as $\theta \rightarrow 1$, the noisy-OR network becomes more ``noisy''.
While, as $\theta \rightarrow 0$, the top level nodes (nodes with no parents) take 
the value 1 with low probability, in which case the child nodes take values 0 with high probability.
\end{remark}

\paragraph{Logistic Regression BNs.}
For logistic regression BNs, the nodes are assumed to be binary valued, i.e., $X_i \in \Xc = \Set{0, 1}$ for all $i$.
Each node in the network can be thought of as being classified as ``$0$'' or ``$1$'' depending
on some linear combination of the values of its parents. 
The parameters for the $i$-th conditional distribution are $\Theta_i = \wv_i$, where 
the vectors $\wv_i$ are assumed to have bounded $\ell_1$ norm, i.e., $\wv_i \in \R^{\Abs{\pii}} \Land \NormI{\wv_i} \leq \wmax^1$,
for some constant $\wmax^1$.
The conditional distribution of the nodes are given as:
\begin{align}
\Pf_i(\pii(G), \Theta_i) = \mathcal{B} \left(\sigma\left( \Inner{\Xpii}{\wv_i}\right) \right),\,
\Pf_i(\varnothing, \Theta_i) = \mathcal{B}(\nicefrac{1}{2}), \label{eq:logistic}
\end{align}
where $\mathcal{B}$ is the Bernoulli distribution and $\sigma(x) = (1 + e^{-x})^{-1}$ is the sigmoid function.
The following lemma upper bounds the mutual information for logistic regression BNs.
\begin{lemma}[Mutual Information bound for Logistic regression networks]
\label{lemma:mi_bound_logistic}
For Logistic regression BNs we have:
\begin{align*}
\klmax \leq \nicefrac{\wmax^1}{2}, &&
\sup_{\Theta \in \Pms(\mathcal{G})} \MI(\Data, G | \Theta) \leq \nicefrac{(nm\wmax^1)}{2},
\end{align*}
where $\wmax^1 \defeq \sup_{\Theta \in \Pms(\mathcal{G})} \max_{i \in [m]} \NormI{\Theta_i(G)}$.
\end{lemma}
\begin{remark}
Once again, we can instantiate Theorem \ref{thm:main} for the two-layer case and obtain necessary
number of samples for support recovery in logistic regression. We have that the number of samples 
needed for support recovery in logistic regression scales as $\BigOm{\nicefrac{k \log(m)}{\wmax^1}}$.
In the regime that $\wmax^1 \leq \nicefrac{1}{k}$, the necessary number of samples scales as $\BigOm{k^2 \log m}$.
Ravikumar et al. \cite{Ravikumar2010} studied support recovery in logistic regression in the
context of learning sparse Ising models. The upper bound of $\BigO{k^2 \log m}$ in Proposition 1 
in \cite{Ravikumar2010}, is thus information-theoretically optimal.
\end{remark}

\paragraph{Concluding Remarks}
\label{sec:conclusion}

An important direction for future work is to study the information-theoretic 
limits of both structure and parameter recovery of BNs. However,
the analysis for that situation is complicated by the fact that one has to
come up with an appropriate joint distribution on the structures and parameters
of the ensembles. While it is possible to do so for BNs with
specific conditional distributions, we anticipate that coming up with general
results for BNs would be hard, if at all possible. Also of complimentary
interest is the problem of obtaining sharp thresholds for structure learning of Bayesian
networks. However, such analysis might also need to be done on a case-by-case basis
for specific BNs.


%
\begin{small}
\bibliographystyle{unsrt}
\bibliography{paper}
\end{small}

\clearpage
\begin{appendices}
\label{sec:appendix}
\section{Comparison with Markov Random Fields.} 
\label{app:comp_with_mrfs}
While there has been a lot of prior work on
determining the information-theoretic limits of structure recovery in Markov random fields (MRFs),
which are undirected graphical models, characterizing the information-theoretic limits of learning
BNs (directed models) is important in its own right for the following reasons.
First, unlike MRFs where the undirected graph corresponding to a dependence structure is uniquely determined,
multiple DAG structures can encode the same dependence structure in BNs. Therefore, one
has to reason about Markov equivalent DAG structures in order to characterize the information-theoretic
limits of structure recovery in BNs. Second, the complexity of learning MRFs
is characterized in terms of parameters of the joint distribution over nodes,
which in turn relates to the overall graph structure, while the complexity of learning
BNs is characterized by parameters of local conditional distributions of the nodes. 
The latter presents a technical challenge, as shown in the paper,
when the marginal or joint distribution of the nodes in a BN do not have a closed form solution. 

A recurring theme in the available literature on information-theoretic limits of learning MRFs,
 is to construct ensembles of MRFs that are hard to learn
and then use the Fano's inequality to lower bound the estimation error by treating the
inference procedure as a communication channel.
Santhanam and Wainwright \cite{santhanam2012information} obtained necessary and sufficient conditions
for learning pairwise binary MRFs. The necessary and sufficient conditions on the number
of samples scaled as $\BigO{k^2 \log m}$ and $\BigO{k^3 \log m}$ respectively, where $k$
is the maximum node degree. Information theoretic limits of learning Gaussian MRFs was studied by Wang et al. \cite{Wang2010} and
for walk-summable Gaussian networks, by Anandkumar et al. \cite{anandkumar2012high}.
In \cite{anandkumar2012annals}, Anandkumar et al. obtain a necessary condition of $\Omega(c \log m)$
for structure learning of Erd\H{o}s-R\'{e}nyi random Ising models, where $c$ is the average node degree.

\section{Proofs of Main Results}
\label{app:proofs_main_results}
\begin{proof}[Proof of Theorem \ref{thm:fano} (Fano's  inequality extension)]
Let,
\begin{align*}
	E \defeq \left\{ \begin{array}{cc}
		1 & \text{if $X \neq \what{X}$} \\
		0 & \text{othewise}
	\end{array}\right.,
\end{align*}
and $\err \defeq \Prob{X \neq \what{X}}$.
Then using the chain rule for entropy we can expand the conditional entropy $H(E, X | \what{X}, W)$ as follows:
\begin{align}
H(E, X | \what{X}, W) 
	&= H(E | X, \what{X}, W) + H(X | \what{X}, W) \label{eq:fano1}\\
	&= H(X | E, \what{X}, W) + H(E | \what{X}, W) \label{eq:fano2}
\end{align}
Next, we bound each of the terms in \eqref{eq:fano1} and \eqref{eq:fano2}.
$H(E | X, \what{X}, W) = 0$ because $E$ is a deterministic function of $X$ and $\what{X}$.
Moreover, since conditioning reduces entropy, we have that $H(E | \what{X}, W) \leq H(E) = H(\err)$.
Using the same arguments we have the following upper bound on $H(X | E, \what{X}, W)$:
\begin{align}
H(X | E, \what{X}, W) 
	&= \err H(X | E = 1, \what{X}, W) + (1 - \err) H(X | E = 0, \what{X}, W)  \notag \\
	&\leq \err H(X|W) \label{eq:fano3}
\end{align}
Next, we show that $\MI(\what{X}; X | W) \leq \MI(\what{X}; Y | W)$, which can be thought
of as the conditional data processing inequality. Using the chain rule of mutual information
we have that
\begin{align*}
\MI(Y, \what{X}; X | W) 
	&= \MI(\what{X}; X | Y, W) + \MI(Y; X | W) \\
	&= \MI(Y; X | \what{X}, W) + \MI(\what{X}; X | W).
\shortintertext{Since, conditioned on $Y$, $X$ and $\what{X}$
are independent. We have $\MI(\what{X}; X | Y, W) = 0$.}
\implies	 \MI(Y; X | W) &= \MI(Y; X | \what{X}, W) + \MI(\what{X}; X | W) \\
\implies \MI(Y; X | W) &\geq \MI(\what{X}; X | W).
\end{align*}
Therefore, we can bound $H(X | \what{X}, W)$ as follows:
\begin{align}
H(X | \what{X}, W) 
	= H(X | W) - \MI(\what{X}; X | W)  
	\geq H(X | W) - \MI(Y; X | W). \label{eq:fano4}
\end{align}
Combining \eqref{eq:fano1},\eqref{eq:fano2},\eqref{eq:fano3} and \eqref{eq:fano4}, we get:
\begin{align}
&H(X | W) - \MI(Y; X | W) \leq H(\err) + \err H(X | W) \notag \\
\implies &H(X | W) - \MI(Y; X | W) \leq \log 2 + \err H(X | W) \notag \\
\implies &\err \geq 1 - \frac{\MI(Y; X | W) + \log 2}{H(X | W)} 
\end{align}
Now if $X$ and $W$ are independent then $H(X|W) = H(X)$. 
Denoting the joint distribution over $X,Y,W$ by $\Pf_{X,Y,W}$,
the conditional distribution of $X,Y$ given $W$ by $\Pf_{X,Y|W}$ and so on;
the final claim follows from  bounding the term $\MI(Y; X | W)$ as follows:
\begin{align*}
\MI(Y; X | W) &= \Exp{\Pf_{X,Y,W}}{\log \frac{\Pf_{X, Y | W}}{\Pf_{X | W}\Pf_{Y|W}}} 
	= \Exp{\Pf_W}{\Exp{\Pf_{X,Y|W = w}}{\log \frac{\Pf_{X, Y | W = w}}{\Pf_{X | W = w}\Pf_{Y|W = w}}}} \\
	&\leq \sup_{w \in \mathcal{W}} \MI(X; Y| W = w).
\end{align*}
\end{proof}
\begin{proof}[Proof of Lemma \ref{lemma:gsrm_size}]
First, we briefly review Steinsky's method for counting essential DAGs, which in turn
is based upon Robinson's \cite{Robinson1977} method for counting labeled DAGs.
The main idea behind Steinsky's method is to 
split the set of essential DAGs into overlapping subsets with different terminal vertices ---
vertices with out-degree 0. Let $A_i \subset \Gsrm$ be the set of essential DAGs 
where the $i$-th node is a terminal node. %
Using the inclusion-exclusion principle, the number of essential DAGs is given as follows:
\begin{align}
&c_{m} \defeq \Abs{\Gsrm} = \Abs{A_{1} \union \ldots \union A_{m}} \notag\\ 
&\quad= \sum_{s=1}^m (-1)^{(s+1)} 
	 \sum_{\mathclap{1 \leq i_1 \leq \ldots \leq i_s \leq m}}
	 	 \Abs{A_{i_1} \intersection \ldots \intersection A_{i_s}}. \label{eq:essn_dags_recurrence}
\end{align}
Now, consider the term $\Abs{A_1 \intersection \ldots \intersection A_{m - 1}}$,
i.e., number of essential DAGs where nodes $[m-1]$ are terminal nodes.
The number of ways of adding the $m$-th vertex as a terminal vertex to
an arbitrary essential DAG on the nodes $[m-1]$ is: $2^{m - 1} - (m - 1)$.
The term $m - 1$ needs to be subtracted to account for edges that are not protected.
Therefore, $c_m$ is given by the following recurrence relation:
\begin{align}
	c_m = \sum_{s = 1}^{m} (-1)^{s + 1} {m \choose s} (2^{m - s} - (m - s))^s c_{m - s},
\end{align}
where $c_0 = 1$.
Using Bonferroni's inequalities we can upper bound $c_m$ as follows:
\begin{align}
&c_m \leq m (2^{m - 1} - (m - 1)) c_{m - 1} \leq m 2^{m - 1} c_{m - 1} \notag \\
	&\quad \leq m!\, 2^{\nicefrac{m(m - 1)}{2}}. \label{eq:count_ub}
\end{align}
Using Bonferroni's inequalities to lower bound $c_m$ produces recurrence relations
that have no closed form solution. Therefore, we lower bound $c_m$ in \eqref{eq:essn_dags_recurrence} as follows:
\begin{align}
c_m &= \Abs{A_{1} \union \ldots \union A_{m}} \geq \max_i \Abs{A_i} \notag \\
	&= (2^{m - 1} - (m - 1)) c_{m - 1} \geq 2^{m - 2} c_{m - 1} \notag \\
&\geq 2^{(\nicefrac{m(m - 3)}{2}) + 1}. \label{eq:count_lb}
\end{align}
\end{proof}

\begin{proof}[Proof of Lemma \ref{lemma:gsrmk_size}]
In this case, $A_i$ is the set of essential DAGs where the $i$-th node is a terminal node
and all nodes have at most $k$ parents. Once again, using the inclusion-exclusion principle,
the number of essential DAGs with at most $k$ parents is given as follows:
\begin{align}
c_{m,k} \defeq \Abs{\Gsmk} = \sum_{s=1}^m (-1)^{(s+1)}
	 \sum_{\mathclap{1 \leq i_1 \leq \ldots \leq i_s \leq m}} \Abs{A_{i_1} \intersection \ldots \intersection A_{i_s}}
\end{align}
Now, consider the term $\Abs{A_1, \intersection \ldots \intersection, A_s}$,
i.e. number of essential DAGs where nodes $\{1, \ldots, s\}$ are terminal nodes.
Let $G_0$ be an arbitrary essential DAG over nodes $\{s + 1, \ldots, m\}$,
where each node has at most $k$ parents. Let $u \in \{1, \ldots, s\}$ and $v \in \{s + 1, \ldots, m\}$
be arbitrary nodes. 
Let $G_1$ be the new graph, formed by connecting $u$ to $G_0$.
The edge $v \rightarrow u$ is not protected, or covered, in $G_1$ if $\Par{G_0}{v} = \Par{G_1}{u} \setminus \{v\}$.
For each node $v \in \{s + 1, \ldots, m\}$, there is exactly one configuration in which the edge $v \rightarrow u$
is covered, i.e., when we set the parents of $u$ to be $\Par{G_0}{v} \union \{v\}$; unless $\Abs{\Par{G_0}{v}} = k$,
in which case $v \rightarrow u$ is always protected. Let $\kappa(G_0)$ be the number of nodes in $G_0$
that have less than $k$ parents. Then the number of ways of adding a terminal vertex $u$ to $G_0$ is:
$\sum_{i=0}^{k} {m - s \choose i} - \kappa(G_0)$ when $(m - s) > k$, and $2^k - k$ when $(m - s) \leq k$.
We can simply bound $\kappa(G_0)$ by: $0 \leq \kappa(G_0) \leq m - s$. This gives a lower bound on
the number of ways to add a terminal vertex to $G_0$ as: 
$\sum_{i=0}^{k} \left\{{m - s \choose i} - \frac{(m - s)}{k+1}\right\} \geq \sum_{i=0}^{k} {m - s - 1 \choose i}$.
Using the fact that $\max_{i=1}^m \Abs{A_i} \leq c_{m,k} \leq \sum_{i=1}^m \Abs{A_i}$, we
get the following recurrence relation for upper and lower bounds on the number of essential DAGs
with at most $k$ parents:
\begin{align}
c_{m, k} &\leq m \left(\sum_{i=0}^k {m - 1 \choose i} \right) c_{m-1, k} \\
c_{m, k} &\geq \left(\sum_{i=0}^k {m - 2 \choose i} \right) c_{m-1, k},
\end{align}
where from Lemma \ref{lemma:gsrm_size}
we have that $2^{(\nicefrac{k(k - 3)}{2}) + 1} \leq  c_{k, k}  \leq k!\, 2^{\nicefrac{k(k - 1)}{2}}$.
Thus, we can upper bound $c_{m,k}$ as follows:
\begin{align}
c_{m,k} \leq m!\, 2^{\nicefrac{k(k - 1)}{2}} \prod_{j=k+1}^{m-1} \left(\sum_{i=0}^k {j \choose i} \right) 
 \label{eq:c_mk_ub}
\end{align}
Similarly, we can lower bound $c_{m,k}$ as follows:
\begin{align}
c_{m,k} \geq 2^{(\nicefrac{k(k - 3)}{2}) + 1} \prod_{j=k+1}^{m-1} \left(\sum_{i=0}^k {j - 1 \choose i} \right) 
\label{eq:c_mk_lb}
\end{align}
Finally, using \eqref{eq:c_mk_lb}, we lower bound $\log c_{m, k}$ as follows:
\begin{align*}
\log c_{m,k} \geq ((\nicefrac{k(k - 3)}{2}) + 1) \log 2 + \sum_{j=k+1}^{m-1} \log 
	\left(\sum_{i=0}^k {j - 1 \choose i} \right) 
\end{align*}
The second term in the above equation is lower bounded as follows:
\begin{align*}
\sum_{j=k+1}^{m-1} \log \left(\sum_{i=0}^k {j - 1 \choose i} \right)
	&\geq \sum_{j=k}^{m-2} \log \left(\sum_{i=0}^k {j \choose i} \right) 
		\geq \sum_{j=k}^{m-2} \log {j \choose k} \geq \sum_{j=k}^{m-2} \log \left(\frac{j}{k}\right)^k \\
	&\geq k\left\{\log (m-2)! -\log k! - (m-k-2) \log k \right\}.
\end{align*}
\end{proof}

\begin{proof}[Proof Lemma \ref{lemma:ls_size}]
For layered non-sparse BNs, the number of possible
choices for parents of a node in layer $i$ is $2^{m_{i + 1}}$.
Therefore, the total number of non-sparse Bayesian networks
is given as $\prod_{i=1}^{l - 1} (2^{m_{i + 1}})^{m_i}$.
Similarly, for the sparse case, the number of possible
choices for parents of a node in layer $i$ is $\sum_{j=0}^k {m_{i + 1} \choose j}$.
Therefore, the total number of sparse Bayesian networks is 
$\prod_{i=1}^{l - 1} \left[\sum_{j=0}^k {m_{i + 1} \choose j}\right]^{m_i}$.
\end{proof}

\begin{proof}[Proof of Lemma \ref{lemma:mi_bound_general}]
Let $c \defeq \Abs{\mathcal{G}}$,
for some ensemble of DAGs $\mathcal{G}$. Denoting the conditional distribution of the 
data given a specific instance of the
parameters $\Theta$ by $\Pf_{\Data| \Theta}$, we have:
\begin{align}
\sup_{\Theta \in \Pms(\mathcal{G})} \MI(\Data; G | \Theta) 
	&= \sup_{\Theta \in \Pms(\mathcal{G})} \frac{1}{c} \sum_{G \in \mathcal{G}} \KL{\Pf_{\Data | G, \Theta}}{\Pf_{\Data| \Theta}},
	\label{eq:mi_1}
\end{align}
where in $\KL{\Pf_{\Data | G, \Theta}}{\Pf_{\Data| \Theta}}$, 
$\Theta$ and $G$ are specific instances and not random variables.
For any distribution $\Qf$ over $\Data$, we can rewrite $\KL{\Pf_{\Data | G, \Theta}}{\Pf_{\Data| \Theta}}$
as follows:
\begin{align}
&\KL{\Pf_{\Data | G, \Theta}}{\Pf_{\Data| \Theta}} 
	=  \Exp{\Data}{\log \frac{\Pf_{\Data | G, \Theta}}{\Qf} \frac{\Qf}{\Pf_{\Data | \Theta}}} \notag \\
	&\quad= \KL{\Pf_{\Data | G, \Theta}}{\Qf} - \Exp{\Data}{\log \frac{\Pf_{\Data | \Theta}}{\Qf}}, \label{eq:mi_2}
\end{align}
where the expectation $\Exp{\Data}{.}$ is with respect to the distribution $\Pf_{\Data | G, \Theta}$.
Now, $\Exp{\Data}{\log \frac{\Pf_{\Data | \Theta}}{\Qf}}$ can be written as follows:
\begin{align}
&\sum_{G \in \mathcal{G}} \Exp{\Data}{\log \frac{\Pf_{\Data | \Theta}}{\Qf}} 
	= \sum_{G \in \mathcal{G}} \sum_{\Data} \Prob{\Data | G, \Theta} \log \frac{\Pf_{\Data | \Theta}}{\Qf} \notag \\
        &\quad= c \sum_{\Data} \sum_{G \in \mathcal{G}} \Prob{G} \Prob{\Data | G, \Theta} \log \frac{\Pf_{\Data | \Theta}}{\Qf} \notag \\
        &\quad= c \KL{\Pf_{\Data | \Theta}}{\Qf}, \label{eq:mi_3}
\end{align}
where, once again, we emphasize that in $\KL{\Pf_{\Data | \Theta}}{\Qf}$, $\Theta$ is a particular instance of the parameters and
not a random variable. Combining $\eqref{eq:mi_1}, \eqref{eq:mi_2}$ and $\eqref{eq:mi_3}$,
and using the fact that $\KL{\Pf_{\Data | \Theta}}{\Qf}> 0$, we get
\begin{align}
\sup_{\Theta \in \Pms(\mathcal{G})} \MI(\Data; G | \Theta) 
	&\leq \sup_{\Theta \in \Pms(\mathcal{G})} \frac{1}{c} \sum_{G \in \mathcal{G}} 
		\KL{\Pf_{\Data | G, \Theta}}{\Qf}
\end{align}
\end{proof}

\begin{proof}[Proof of Lemma \ref{lemma:kl_bound_exp} (KL bound for exponential family)]
\begin{align}
\KL{\Pf_1}{\Pf_2} &= \etav_1^T \Exp{X}{\ST(x) | \etav_1} - \psi(\etav_1) - \etav_2^T \Exp{X}{\ST(x) | \etav_1}
                + \psi(\etav_2),
\end{align}
where for computing the expected sufficient statistic, $\Exp{X}{\ST(x) | \etav_1}$,
we take the expectation with respect to the distribution parameterized by  $\etav_1$.
Now, from the mean value theorem we have that
\begin{align*}
\psi(\etav_2) - \psi(\etav_1) &= \Grad\psi(\alpha \etav_2 + (1 - \alpha)\etav_1)^T[\etav_2 - \etav_1] \\
	&= (\etav_2 - \etav_1)^T\Exp{X}{\ST(x) | \alpha \etav_2 + (1 - \alpha)\etav_1}
\end{align*}
for some $\alpha \in [0, 1]$. Then we have that,
\begin{align}
\KL{\Pf_1}{\Pf_2}  
&= \EST(\etav_1)^T[\etav_1 - \etav_2] + (\etav_2 - \etav_1)^T \EST(\alpha \etav_2 + (1 - \alpha) \etav_1) \notag \\
&= (\etav_1 - \etav_2)^T\{ \EST(\etav_1) - \EST(\alpha \etav_2 + (1 - \alpha) \etav_1) \} \label{eq:kl_ith}
\end{align}
Since the function $\EST$ is the gradient of the convex function $\psi$, it is monotonic.
Therefore, the function $\EST(\alpha \etav_2 + (1 - \alpha)\etav_1)$ takes the maximum value at the end
points $\alpha = 0$ or at $\alpha = 1$. Assuming $\EST$ is maximized at $\alpha = 0$,
the $i$-th KL divergence term can be upper bound using \eqref{eq:kl_ith} as:
\begin{align*}
0 \leq \KL{\Pf_1}{\Pf_2} \leq (\etav_1 - \etav_2)^T\{\EST(\etav_1) - \EST(\etav_2)\}
\end{align*}
On the other hand, assuming $\EST$ is maximized at $\alpha = 1$, the
$i$-th KL divergence term can be upper bound using \eqref{eq:kl_ith} as:
\begin{align*}
0 \leq \KL{\Pf_1}{\Pf_2} \leq 0.
\end{align*}
Therefore, clearly, $\EST(\alpha \etav_2 + (1 - \alpha)\etav_1)$ is maximized at $\alpha = 0$.
\end{proof}
\begin{proof}[Proof of Theorem \ref{thm:main}]
Setting the measure $\Pf_{\Gs}$ to be the uniform over $\Gs$,
and using the Fano's inequality from Theorem \ref{thm:fano} and 
the mutual information bound from Lemma \ref{lemma:mi_bound}, combined with our 
Assumption \ref{ass:kl}, we can bound the estimation error as follows:
\begin{align*}
\err \geq 1 - \frac{nm\klmax + \log 2}{\log \Abs{\Gs}}.
\end{align*}
Then by using the lower bounds on the number of DAG structures in each of the
ensembles from Lemmas \ref{lemma:gsrm_size}, \ref{lemma:gsrmk_size} and \ref{lemma:ls_size},
and setting $\err$ to $\nicefrac{1}{2}$, we prove our claim.
\end{proof}
\section{Proofs of Results for Commonly Used Bayesian Networks}
\label{app:proof_common_bns}
\begin{proof}[Proof of Lemma \ref{lemma:cpt_mi_bound} (Mutual Information bound for CPT networks)]
For CPT, the mutual information bound  
is representative of the case when we do not have a closed form solution for the
marginal and joint distributions; yet, we can easily bound $\Dl(\etai, \etaz)$ through
a simple application of the Cauchy-Schwartz inequality, and obtain tighter bounds than the
naive $\BigO{mn \log v}$ bound on the mutual information $\MI(\Data; G | \Theta)$.
The sufficient statistics and the natural parameter for the
categorical distribution is given as follows:
\begin{align*}
\ST(x) &= \left(\Ind{x = j}\right)_{j=1}^{v} && 
\etai(\Xpii, \Theta_i) = \left(\log \theta_{ij}(\Xpii) \right)_{j=1}^{v}. 
\end{align*}
Therefore, the expected sufficient statistic $\EST(\etai) = \Theta_i(\Xpii)$. From that we 
get the following upper bound
\begin{align*}
\sup_{\Theta \in \Pms(G)} \Exp{\Xpii}{\Dl(\etai, \etaz)}
&= \sup_{\Theta \in \Pms(G)} \Exp{\Xpii}{(\etai - \etaz)^T\{\EST(\etai) - \EST(\etaz)\}} \\
&\leq \sup_{\Theta \in \Pms(G)} \Exp{\Xpii}{\NormInfty{\etai - \etaz}\NormI{\EST(\etai) - \EST(\etaz)}} \\
&\leq 4 \log(\nicefrac{1}{\thetamin} ),
\end{align*}
where in the above we used the Cauchy-Schwartz inequality followed by the fact
$\NormI{\Theta_i(\xv)} = 1,\, \forall \xv \in \Xc$.
\end{proof}

\begin{proof}[Proof of Lemma \ref{lemma:mi_bound_gaussian} (Mutual Information bound for Gaussian)]
This exemplifies 
the case where we have closed form solutions for the joint and marginal distributions, which
in this case is Gaussian, and we can compute the expected value of $\Dl(\etai, \etaz)$.
The sufficient statistics and natural parameter for the $i$-th conditional distribution
are given as follows:
\begin{align*}
\ST(X_i) = \frac{X_i}{\nicefrac{\sigma}{\sqrt{2}}}, && \etai = \frac{\mu_i}{\nicefrac{\sigma}{\sqrt{2}}}.
\end{align*}
Also note that, $\forall i \in [m],$ the marginal expectation $\Exp{}{X_i} \leq \mu \NormII{\wv_i}$.
Therefore, we have that $\Exp{\Xpii}{\Dl(\etai, \etaz)} = \Exp{\Xpii}{\nicefrac{2(\mu_i - \mu)^2}{\sigma^2}} = 
\nicefrac{2(\Exp{\Xpii}{(\mu_i - \mu)}^2 + \Var{\Xpii}{\mu_i})}{\sigma^2} 
	\leq \nicefrac{2(\mu^2(\NormII{\wv_i} - 1)^2 + \Var{\Xpii}{\mu_i})}{\sigma^2}$.
Hence, we need to upper bound  $\Var{\Xpii}{\mu_i}$ in order to upper bound $\Exp{\Xpii}{\Dl(.)}$. 
Let $(i)_G \in [m]$ be the $i$-th node in the topological order defined by the graph $G$.
We use the shorthand notation $(i)$, when it is clear from context that the $i$-th node
in the topological ordering is intended.
Now, from the properties of the Gaussian distribution we know that if the conditional
distributions are all Gaussian, then the joint distribution over any subset of $\Xs$ is Gaussian as well.
Let $\mat{\Sigma} \in \R^{m \times m}$ be the covariance matrix for the joint distribution over $\Xs$, 
and similarly $\mat{\Sigma}_{(i)} \in \R^{i \times i}$ denote the covariance matrix for the joint distribution over 
variables $\Set{X_{(1)}, \ldots, X_{(i)}}$. 
Let $\wbv_{(i)} \in \R^{i - 1}$ be the weight vector defined as follows:
\begin{align*}
	\forall j \in [i - 1],\, (\overline{w}_{(i)})_j &= \left\{\begin{array}{cc}
		0 & \text{if $j \notin \mathsf{\pi}_{(i)}$,} \\
		(w_{(i)})_j & \text{otherwise.} 
	\end{array}\right. 
\end{align*} 
Note that $\NormII{\wbv_{(i)}} = \NormII{\wv_{(i)}} \leq \nicefrac{1}{\sqrt{2(i - 1)}}$.
Then, we have that $\Var{}{\mu_{(i)}} = \wbv_{(i)_G}^T \mat{\Sigma}_{(i-1)_G} \wbv_{(i)_G}$
and $\Var{}{X_{(i)}} = \Var{}{\mu_{(i)}} + \nicefrac{\sigma^2}{2}$.
Also, for any $j \in [i - 1]$, we have that $\Cov{}{X_{(i)} X_{(j)}} = \wbv_{(i)_G}^T(\mat{\Sigma}_{(i-1)_G})_{*,j}$,
where $(\mat{\Sigma}_{(i-1)_G})_{*,j}$ is the $j$-th column of the matrix $\mat{\Sigma}_{(i-1)}$. Therefore,
the covariance matrix $\mat{\Sigma}_{(i)_G}$ can be written as follows:
\begin{align*}
	\mat{\Sigma}_{(i)_G} = \matrx{\mat{\Sigma}_{(i-1)_G} & \mat{\Sigma}_{(i-1)_G} \wbv_{(i)_G} \\
		\wbv_{(i)_G}^T \mat{\Sigma}_{(i-1)_G} & \wbv_{(i)_G}^T \mat{\Sigma}_{(i-1)_G} \wbv_{(i)_G} + \nicefrac{\sigma^2}{2}},
\end{align*}
where $\mat{\Sigma}_{(1)} \in \R^{1 \times 1} = [[\nicefrac{\sigma^2}{2}]]$.
Since $\mat{\Sigma}_{(i)}$ is positive definite, we have that 
$\eigmax(\mat{\Sigma}_{(i)}) \leq \eigmax(\mat{\Sigma}_{(i - 1)}) + 
\wbv_{(i)}^T \mat{\Sigma}_{(i-1)} \wbv_{(i)} + \nicefrac{\sigma^2}{2}$. Next,
we prove, by induction, that $\eigmax(\mat{\Sigma}_{(i)}) \leq i \sigma^2$. First,
note that the base case holds, i.e., $\eigmax(\mat{\Sigma}_{(1)}) = \nicefrac{\sigma^2}{2} \leq \sigma^2$.
Now assume, that $\eigmax(\mat{\Sigma}_{(i - 1)}) \leq (i - 1)\sigma^2$. Then, we have:
\begin{align*}
\eigmax(\mat{\Sigma}_{(i)}) 
	&\leq \eigmax(\mat{\Sigma}_{(i - 1)}) + 
		\wbv_{(i)}^T \mat{\Sigma}_{(i-1)} \wbv_{(i)} + \frac{\sigma^2}{2} \\
	&\leq (i - 1)\sigma^2 + \NormII{\wbv_{(i)}}^2 (i - 1)\sigma^2 + \frac{\sigma^2}{2} \\
	&\leq (i - 1)\sigma^2 + \frac{1}{2(i-1)}(i - 1)\sigma^2 + \frac{\sigma^2}{2} \\
	&\leq i \sigma^2.
\end{align*}
Therefore, we can bound the variance of $\mu_i$ as follows:
\begin{align*}
\Var{}{\mu_{(i)}} &= \wbv_{(i)}^T \mat{\Sigma}_{(i-1)} \wbv_{(i)} 
	\leq \NormII{\wbv_{(i)}}^2 \eigmax(\mat{\Sigma}_{(i-1)})  \\
&\leq \frac{1}{2(i - 1)} (i - 1)\sigma^2 = \frac{\sigma^2}{2}.
\end{align*}
Thus, we have that $\Var{}{\mu_{i}} \leq \nicefrac{\sigma^2}{2}$,
$\Exp{\Xpii}{\Dl(\etai, \etaz)} \leq 1 + \nicefrac{2\mu^2 (\NormII{\wv_i} - 1)^2}{\sigma^2}$,
and $\klmax \leq 1 + \nicefrac{(2 \mumax^2 (\wmax^2 + 1))}{\sigmamin^2}$.

\end{proof}

\begin{proof}[Proof of Lemma \ref{lemma:mi_bound_noisy_or} (Mutual Information bound for Noisy-OR)]
The expected sufficient statistics and natural parameter for the Bernoulli distribution
is given as:
\begin{align*}
\EST(\etai) = 1 - \theta_i, && \etai = \log \frac{\theta_i}{1 - \theta_i}.
\end{align*}
Also, $\EST(\etaz) = \theta$ and $\etaz = \log(\nicefrac{(1 - \theta)}{\theta})$.
Using the fact that $\theta^2 \leq \theta_i \leq \theta$, we can bound $\Exp{\Xpii}{\Dl(\etai, \etaz)}$
as follows:
\begin{align*}
\Dl(\etai, \etaz)
	&=\left(\log \frac{\theta_i}{1 - \theta_i} - \log \frac{1 - \theta}{\theta}\right) (1 - \theta_i - \theta) \\
	&\leq \Abs{\log \frac{\theta_i \theta}{(1 - \theta_i) (1 - \theta)}} \\
	&\leq 2 \Abs{\log \frac{\theta}{1 - \theta}}
\end{align*}
Therefore, we have that
$\klmax \leq 2\Abs{\log(\nicefrac{\hat{\theta}}{(1 - \hat{\theta})})}$.

\end{proof}
\begin{proof}[Proof of Lemma \ref{lemma:mi_bound_logistic} (MI bound for Logistic regression networks)]
The expected sufficient statistics and the natural parameter are given as follows:
\begin{align*}
\EST(\etai) = \sigma(\wv_i^T \Xpii),&& \etai = \log \frac{\sigma(\wv_i^T \Xpii)}{1 - \sigma(\wv_i^T \Xpii)} = \wv_i^T \Xpii.
\end{align*}
From the above, we also have that $\EST(\etaz) = \nicefrac{1}{2}$ and $\etaz = 0$. Then, $\klmax$
is bounded as follows:
\begin{align*}
\klmax &= \Exp{\Xpii}{\wv_i^T \Xpii(\sigma(\wv_i^T \Xpii) -\nicefrac{1}{2})} \\
	&\leq \frac{1}{2} \Exp{\Xpii}{\wv_i^T \Xpii} \leq \frac{1}{2} \Exp{\Xpii}{\NormI{\wv_i} \NormInfty{\Xpii}} \\
	&\leq \frac{\NormI{\wv_i}}{2} \leq \frac{\wmax^1}{2}. 
\end{align*}
\end{proof}

\end{appendices}

\end{document}